\documentclass[10pt,twocolumn,letterpaper]{article}

\usepackage{iccv}
\usepackage{times}
\usepackage{epsfig}
\usepackage{graphicx}
\usepackage{amsmath}
\usepackage{amssymb}

\usepackage{algpseudocode}
\usepackage{algorithm}
\usepackage{amsthm}
\usepackage{bm}
\usepackage{mathrsfs}
\usepackage{subfig}
\usepackage{multirow}
\usepackage[export]{adjustbox}

\usepackage[pagebackref=true,breaklinks=true,letterpaper=true,colorlinks,bookmarks=false]{hyperref}

\iccvfinalcopy 

\def\iccvPaperID{1412} 

\newtheorem{theorem}{Theorem}

\usepackage{amsmath,amsfonts,bm}

\def\eqref#1{equation~\ref{#1}}

\def\1{\bm{1}}

\def\vw{{\bm{w}}}
\def\vx{{\bm{x}}}
\def\vy{{\bm{y}}}

\def\mW{{\bm{W}}}

\DeclareMathAlphabet{\mathsfit}{\encodingdefault}{\sfdefault}{m}{sl}
\SetMathAlphabet{\mathsfit}{bold}{\encodingdefault}{\sfdefault}{bx}{n}

\def\sR{{\mathbb{R}}}

\newcommand{\E}{\mathbb{E}}

\def\ourmethod{\textit{HORDE}}

\ificcvfinal\fi
\begin{document}

\title{Metric Learning With HORDE: High-Order Regularizer for Deep Embeddings}

\author{Pierre Jacob$^1$,~
        David Picard$^{1,2}$,~
        Aymeric Histace$^1$,~
        Edouard Klein$^3$~\\
        $^1$ETIS UMR 8051, Universit{\'e} Paris Seine, UCP, ENSEA, CNRS, F-95000, Cergy, France \\
        $^2$LIGM, UMR 8049, \'Ecole des Ponts, UPE, Champs-sur-Marne, France \\
        $^3$C3N, P\^{o}le Judiciaire de la Gendarmerie Nationale, 5 boulevard de l'Hautil, 95000 Cergy, France\\
        \small \{pierre.jacob, picard, aymeric.histace\}@ensea.fr}

\maketitle

\begin{abstract}
Learning an effective similarity measure between image representations is key to the success of recent advances in visual search tasks (e.g. verification or zero-shot learning).
Although the metric learning part is well addressed, this metric is usually computed over the average of the extracted deep features.
This representation is then trained to be discriminative.
However, these deep features tend to be scattered across the feature space.
Consequently, the representations are not robust to outliers, object occlusions, background variations, etc.
In this paper, we tackle this scattering problem with a distribution-aware regularization named \ourmethod\footnote{Code is available at \url{https://github.com/pierre-jacob/ICCV2019-Horde}}.
This regularizer enforces visually-close images to have deep features with the same distribution which are well localized in the feature space.
We provide a theoretical analysis supporting this regularization effect.
We also show the effectiveness of our approach by obtaining state-of-the-art results on 4 well-known datasets (Cub-200-2011, Cars-196, Stanford Online Products and Inshop Clothes Retrieval).

\end{abstract}

\begin{figure*}
    \centering
    \subfloat[Representations using all deep features \label{fig:illustration_full}]{\includegraphics[width=0.32\linewidth]{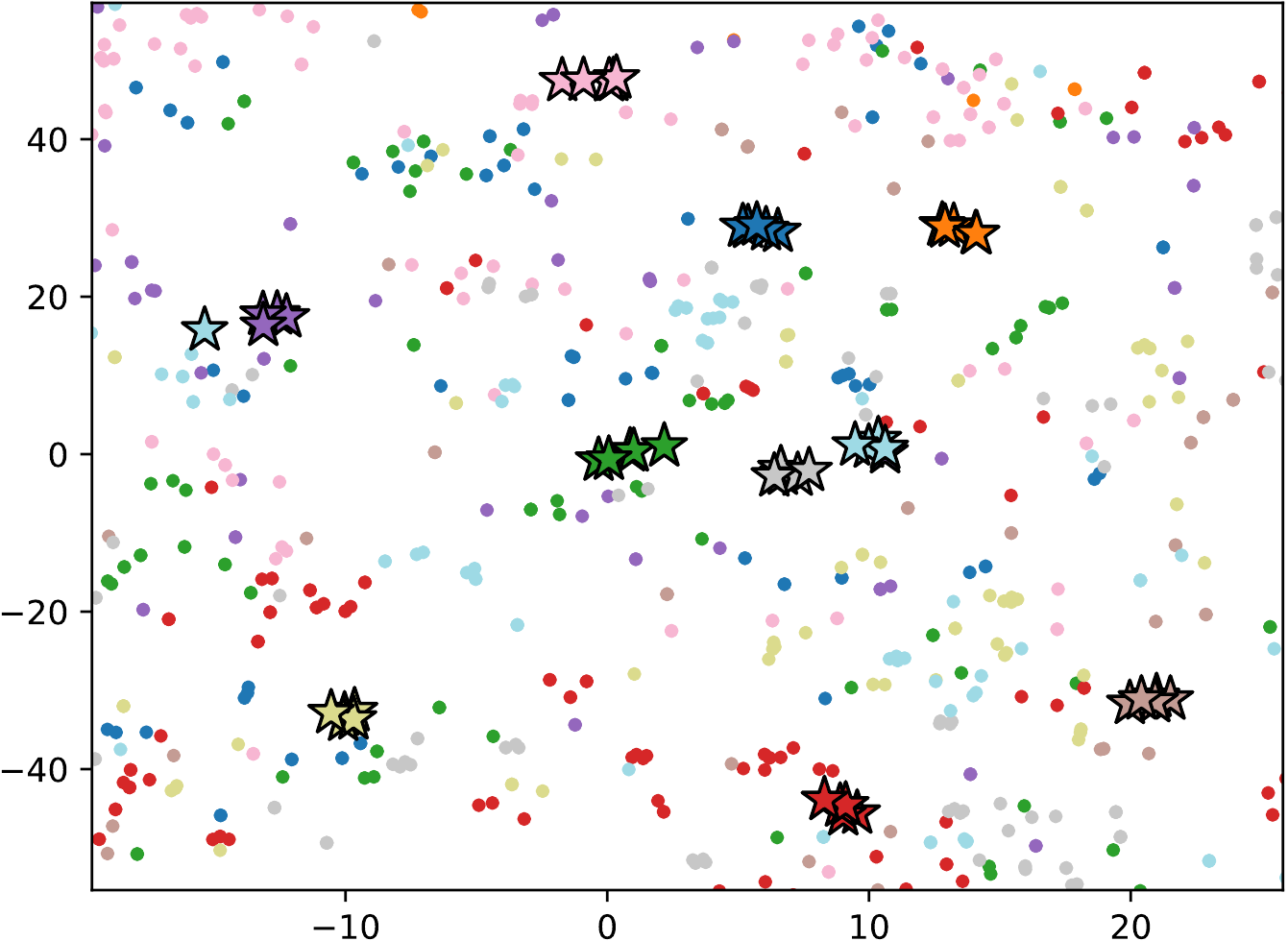}}
    \hfill
    \subfloat[Representations using $1/6$-th of deep features \label{fig:illustration_low}]{\includegraphics[width=0.32\linewidth]{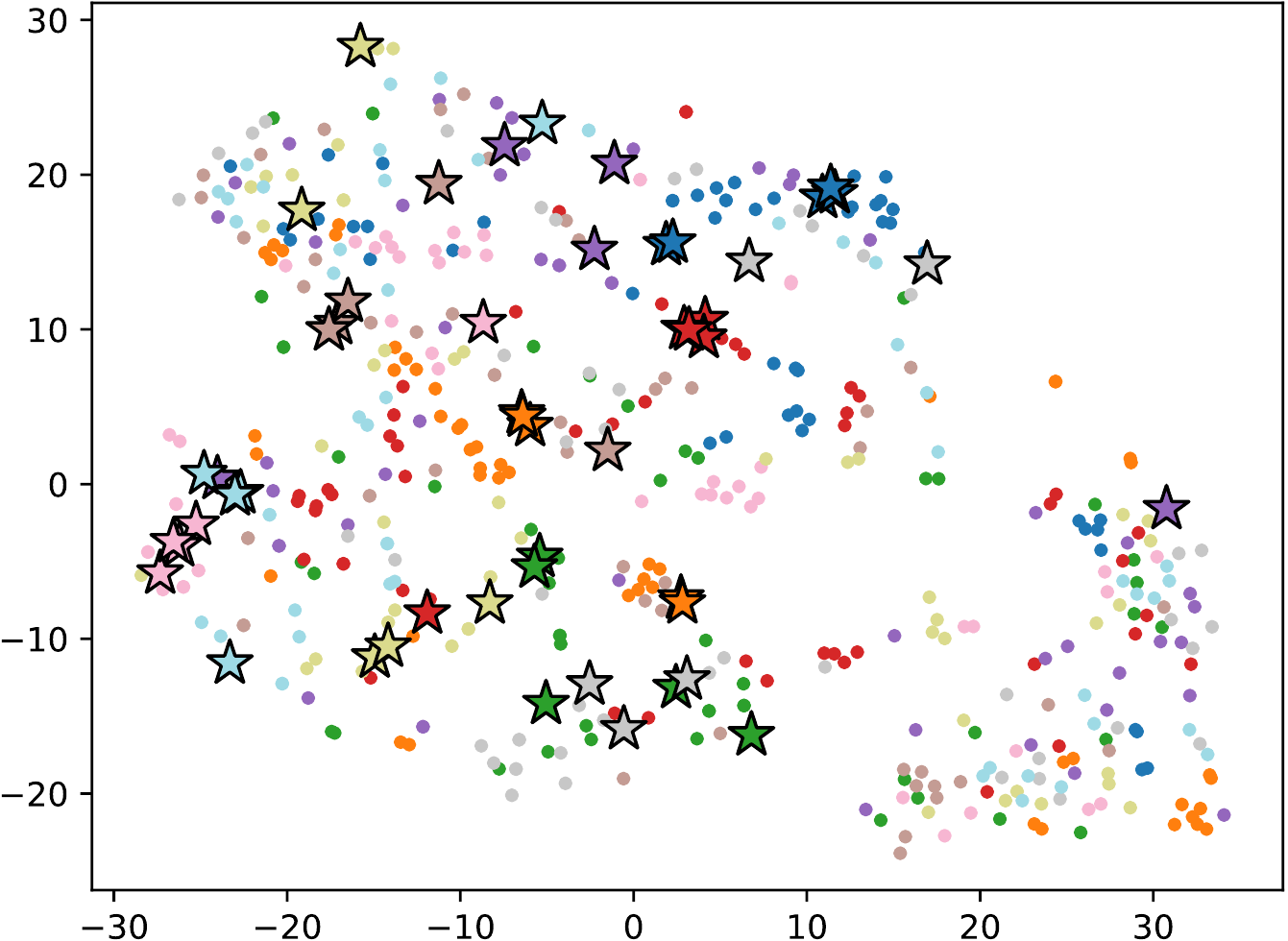}}
    \hfill
    \subfloat[Representations using $1/6$-th of deep features trained with \ourmethod \label{fig:illustration_horde}]{\includegraphics[width=0.32\linewidth]{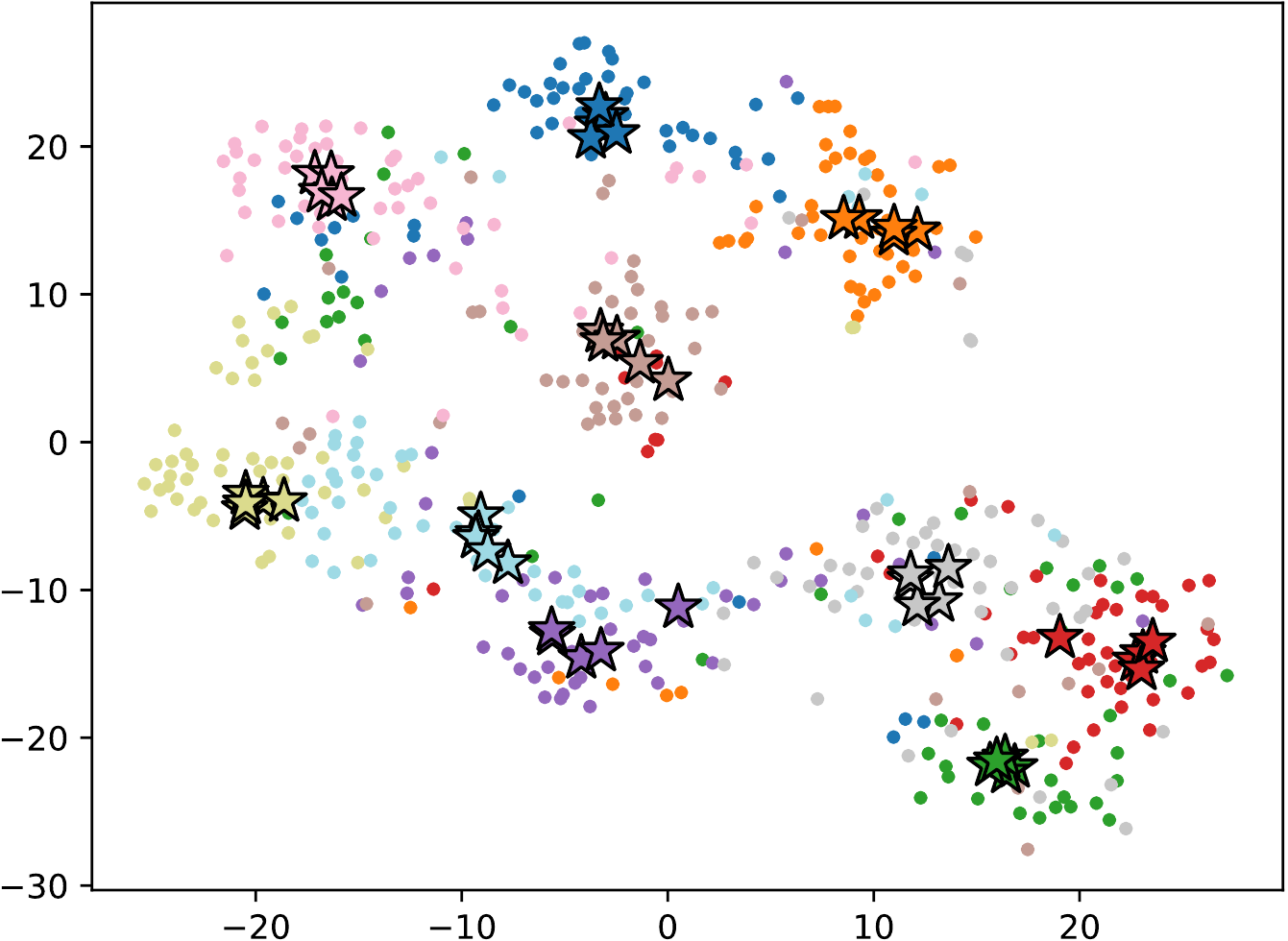}}
    \hfill
    \caption{2D visualizations of representations (stars) and deep features (points) using t-SNE computed from a DML architecture on MNIST dataset with \textbf{one color per class}. Representations and features come from the \textbf{training set}. \autoref{fig:illustration_full} shows discriminative representations but with scattered deep features (Remark the scale of the axes). \autoref{fig:illustration_low} shows representations computed with $1/6$-th of the deep features, leading to a disorganized space. \autoref{fig:illustration_horde} shows the same model trained with \ourmethod : the deep features are well concentrated and the representations computed using $1/6$-th of the deep features are organized according to the classes (best viewed on a computer screen).}
    \label{fig:illustration}
\end{figure*}

\section{Introduction}\label{sec:introduction}
    Deep Metric Learning (DML) is an important yet challenging topic in the Computer Vision community with numerous applications such as visual product search \cite{Liu_2016_CVPR_INSHOP, Song_2016_CVPR}, multi-modal retrieval \cite{Carvalho_2018_SIGIR, Wehrmann_2018_CVPR}, face verification and clustering \cite{Schroff_2015_CVPR}, person or vehicle identification \cite{Liu_2016_CVPR, Zhou_2017_ICCV}.
    To deal with such applications, a DML method aims to learn an embedding space where all the visually-related images (\emph{e.g.}, images of the same car model) are close to each other and dissimilar ones (\emph{e.g.}, images of two cars from the same brand but from different models) are far apart.
    
    Recent contributions in DML can be divided into three categories.
    A first category includes methods that focus on batch construction to maximize the number of pairs or triplets available to compute the similarity (\emph{e.g.}, N-pair loss \cite{Sohn_2016_NIPS}).
    A second category involves the design of loss functions to improve the generalization (\emph{e.g.}, binomial deviance \cite{Ustinova_2016_NIPS}).
    The third category covers ensemble methods that tackle the embedding space diversity (\emph{e.g.}, BIER \cite{Opitz_2017_ICCV}).
    
    This similarity metric is trained jointly with the image representation which is computed using deep neural network architectures such as GoogleNet \cite{Szegedy_2015_CVPR} or BN-Inception \cite{Ioffe_2015_ICML}.
    For all of these networks, the image representations are obtained by the aggregation of the deep features using a Global Average Pooling \cite{Zhou_2016_CVPR}.
    Hence, the deep features are summarized using the sample mean, and the training process makes sure that the sample mean is discriminative enough for the target task.
    
    
    Our insight is that ignoring the characteristics of the deep feature distribution leads to a lack of distinctiveness in the deep features.
    We illustrate this phenomenon in \autoref{fig:illustration}.
    In \autoref{fig:illustration_full}, we train a DML model on MNIST and plot both the deep features and the image representations from a set of images sampled from the training set.
    We observe that the representations are perfectly organized while the deep features are in contrast scattered in the entire space.
    As the representations are obtained using the sample mean only, they are sensitive to outliers or sampling problems (occlusions, illumination, background variation, etc.), which we refer to as the \emph{scattering problem}.
    We illustrate this problem in \autoref{fig:illustration_low} where the representations are computed using the same architecture but by sampling only $1/6$-th of the original deep features.
    As we can see, the resulting representations are no longer correctly organized.
    
    In this paper, we propose \ourmethod, a High-Order Regularizer for Deep Embeddings which tackles this scattering problem.
    By minimizing (resp. maximizing) the distance between high-order moments of the deep feature distributions, this DML regularizer enforces deep feature distributions from similar (resp. dissimilar) images to be nearly identical (resp. to not overlap).
    As illustrated in \autoref{fig:illustration_horde}, our \ourmethod \ regularizer produces well localized features, leading to robust image representations even if they are computed using only $1/6$-th of the original deep features.
    
    Our contributions are the following: First, we propose a High-Order Regularizer for Deep Embeddings (\ourmethod) that reduces the scattering problem and allows the sample mean to be a robust representation.
    We provide a theoretical analysis in which we support this claim by showing that \ourmethod \ is a lower bound of the Wasserstein distance between the deep feature distributions while also being an upper-bound of their Maximum Mean Discrepancy.
    Second, we show that \ourmethod \ consistently improves DML with varying loss functions, even when considering ensemble methods. Using \ourmethod, we are able to obtain state of the art results on four standard DML datasets (Cub-200-2011 \cite{CUB_200_2011}, Cars-196 \cite{CARS_196}, In-Shop Clothes Retrieval \cite{Liu_2016_CVPR_INSHOP} and Stanford Online Products \cite{Song_2016_CVPR}).
    
    The remaining of this paper is organized as follows:
    In \autoref{sec:related_work}, we review recent works on deep metric learning and how our approach differs.
    In \autoref{sec:method_overview}, after an overview of our proposed method, we present the practical implementation of \ourmethod \ as well as a theoretical analysis.
    In \autoref{sec:sota} we compare our proposed architecture with the state-of-the-art on four image retrieval datasets.
    We show the benefit of \ourmethod \ regularization for different loss functions and an ensemble method.
    In \autoref{sec:abation_studies} we conduct extensive experiments to demonstrate the robustness of our regularization and its statistical consistency.

\begin{figure*}
    \centering
    \includegraphics[width=\linewidth]{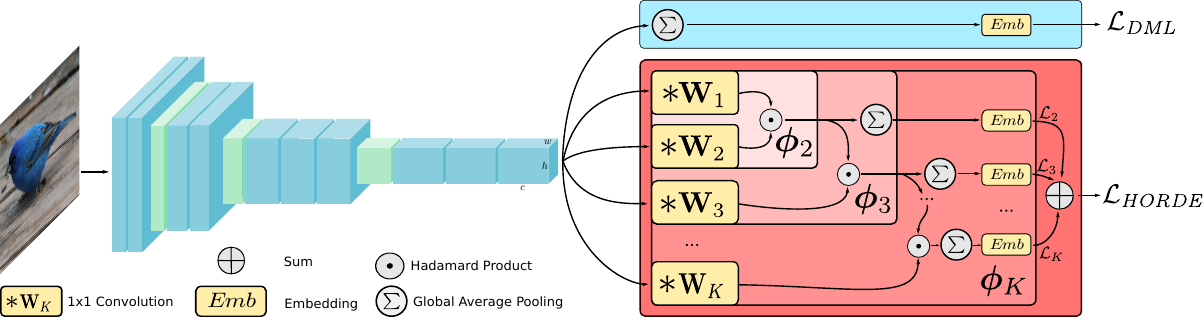}
    \caption{Global overview of our \ourmethod \ architecture. The deep convolutional neural network extracts $h \times w \times c$ deep features. The standard architecture (top blue block) relies on a global average pooling and an embedding before computing the $\mathcal{L}_{\text{DML}}$ loss. The bottom red block is our \ourmethod \ regularizer, composed by the approximation of all high-order moments $\boldsymbol{\phi}_k$,  global average pooling and embeddings before computing the sum of each $\mathcal{L}_k$ loss.}
    \label{fig:horde_architecture}
\end{figure*}
    
\section{Related Work}\label{sec:related_work}
    In DML, we jointly learn the image representations and an embedding in such a way that the Euclidean distance corresponds with the semantic content of the images.
    Current approaches use a pre-trained CNN to produce deep features, then they aggregate these features using Global Average Pooling \cite{Zhou_2016_CVPR}.
    Finally they learn the target representation with a linear projection.
    The whole network is fine-tuned to solve the metric learning task according to three criteria: a loss function, a sampling strategy and an ensemble method.
    
    Regarding the loss function, popular approaches consider pairs \cite{Chopra_CVPR_2005} or triplets \cite{Schroff_2015_CVPR} of similar/dissimilar samples.
    Recent works generalize these loss functions to larger tuples \cite{Chen_2017_CVPR, Song_2016_CVPR, Sohn_2016_NIPS, Ustinova_2016_NIPS} or improve the design \cite{Wang_2018_CVPR, Wang_2017_ICCV, Yu_2018_ECCV}.
    The sampling of the training tuples receive plenty of attention \cite{Song_2016_CVPR, Schroff_2015_CVPR, Sohn_2016_NIPS}, either through mining \cite{Harwood_2017_ICCV, Schroff_2015_CVPR}, proxy based approximations \cite{Movshovitz-Attias_2017_ICCV, Song_2017_CVPR} or hard negative generation \cite{Duan_2018_CVPR, Lin_2018_ECCV}.
    Finally, ensemble methods have recently become an increasingly popular way of improving the performances of DML architectures \cite{Kim_2018_ECCV, Opitz_2017_ICCV, Xuan_2018_ECCV, Yuan_2017_ICCV}.
    Our proposed \ourmethod \ regularizer is a complementary approach.
    We show in \autoref{sec:sota} that it consistently improves these popular DML models.
    
    Recent approaches also consider a distribution analysis for DML \cite{Rippel_2016_ICLR,Lin_2018_ECCV}.
    Contrarily to us, they only consider the distribution of the representations to design a loss function or a hard negative generator but they do not take into account the distribution of the underlying deep features.
    Consequently, they do not address the scattering problem.
    More precisely, Magnet loss \cite{Rippel_2016_ICLR} proposes to better represent a given class manifold by learning a $K$-mode distribution instead of the standard uni-mode assumption.
    To that aim, the per-class distribution is approximated using $K$-means clustering.
    The proposed loss tries to minimize the distance between a representation and its nearest class mode and tries to maximize the distance between all modes of all other classes.
    However, since the magnet loss is directly applied to the sample means of the deep features, it leads to the scattering problem illustrated in \autoref{fig:illustration}.
    In DVML \cite{Lin_2018_ECCV}, the authors assume that the representations follow a per-class Gaussian distribution. They propose to estimate the parameters of these distributions using a variational auto-encoder approach.
    Then, by sampling from a Gaussian distribution with the learned parameters, they are able to generate artificial hard samples to train the network.
    However, no assumption is made on the distribution of the deep features, which leads to the scattering problem illustrated in \autoref{fig:illustration} (see also \cite{Lin_2018_ECCV}, Figure 1).
    In contrast, we show that focusing on the distribution of the deep features reduces the scattering problem and improves the performances of DML architectures.
    
    In the next section, we first give an overview of the proposed \ourmethod \ regularization.
    Then, we describe the practical implementation of the high-order moments computation.
    Finally, we give theoretical insights which support the regularization effect of \ourmethod.

\section{Proposed High-Order Regularizer}\label{sec:method_overview}
    We first give an overview of the proposed method in \autoref{fig:horde_architecture}.
    We start by extracting a deep feature map of size $h \times w \times c$ using a CNN where $h$ and $w$ are the height and width of the feature map and $c$ is the deep features dimension.
    Following standard DML practices, these features are aggregated using a Global Average Pooling to build the image representation and are projected into an embedding space before a similarity-based loss function is computed over these representations (top-right blue box in \autoref{fig:horde_architecture}).
        
    In \ourmethod, we directly optimize the distribution of the deep features by minimizing (respectively maximizing) a distance between the deep feature distributions of similar images (respectively dissimilar images).
    We approximate the deep feature distribution distance by computing high-order moments (bottom-right red box in \autoref{fig:horde_architecture}).
    We recursively approximate the high-order moments and we compute an embedding after each of these approximations.
    Then, we apply a DML loss function on each of these embeddings.
    
        
    \subsection{High-order computation}\label{sec:met_implementation}
        In practice, the computation of high-order moments is very intensive due to their high dimension.
        Furthermore, it has been shown in \cite{Jegou_2012_ECCV, Opitz_2017_ICCV} that an independence assumption over all high-order moment components is unrealistic.
        Hence, we rely on factorization schemes to approximate their computation, such as Random Maclaurin (RM) \cite{Kar_PMLR_2012}.
        The RM algorithm relies on a set of random projectors to approximate the inner product between two high-order moments.
        In the case of the second-order, we sample two independent random vectors $\vw_1, \vw_2 \sim \mathcal{W}$ where $\mathcal{W}$ is a uniform distribution in $\{-1, +1\}$.
        For two non random vectors $\vx$ and $\vy$, the inner product between their second-order moments can be approximated as:
        \begin{align}
            \nonumber \E_{\vw_1, \vw_2 \sim \mathcal{W}}[\phi_2(\vx) \phi_2(\vy)] &= \left<\vx \ ; \ \vy \right>^2 \\
            &= \left<\vx \otimes \vx \ ; \ \vy \otimes \vy \right> 
        \end{align}
        where $\otimes$ is the Kronecker product, $\E_{\vw_1, \vw_2 \sim \mathcal{W}}$ is the expectation over the random vectors $\vw_1$ and $\vw_2$ which follow the distribution $\mathcal{W}$ and $\phi_2(\vx) = \left<\vw_1 \ ; \ \vx \right>\left<\vw_2 \ ; \ \vx \right>$.
        This approach easily holds to estimate any inner product between $K$-th moments:
        \begin{align}
            \nonumber \E_{\vw_{k} \sim \mathcal{W}}[\phi_K(\vx) \phi_K(\vy)] &= \left<\vx \ ; \ \vy \right>^K \\
            &= \left<\underbrace{\vx \otimes \dots \otimes \vx}_{K \ \text{times}}\ ; \underbrace{\vy \otimes \dots \otimes \vy}_{K \ \text{times}} \right> 
        \end{align}
        where $\phi_K(\vx)$ is computed as:
        \begin{align}
            \phi_K(\vx) = \prod_{k=1}^K \left<\vw_k \ ; \ \vx \right>
        \end{align}
        In practice, we approximate the expectation of this quantity by using the sample mean over $d$ sets of these random projectors.
        That is, we sample independent random matrices $\mW_1, \mW_2, ..., \mW_K \in \sR^{c \times d}$ and we compute the vector $\boldsymbol{\phi}_K(\vx) \in \sR^d$ that approximates the $K$-th moments of $\vx$ with the following equation:
        \begin{align}\label{eq:maclaurin}
            \boldsymbol{\phi}_K(\vx) = \left( \mW_1^\top \vx \right) \odot \left( \mW_2^\top \vx \right) \odot \dots \odot \left( \mW_K^\top \vx \right)
        \end{align}
        where $\odot$ is the Hadamard (element-wise) product.
        Thus, the inner product between the $K$-th moments is:
        \begin{align}
            \left<\vx \ ; \ \vy \right>^K \approx \frac{1}{d} \left<\boldsymbol{\phi}_K(\vx) \ ; \ \boldsymbol{\phi}_K(\vy) \right>
        \end{align}
        
        However, Random Maclaurin produces a consistent estimator independently of the analyzed distributions, and thus also encodes non informative high-order moment components.
        To ignore these non-informative components, the projectors $\mW_k$ can be learned from the data.
        However, the high number of parameters in $\mathcal{O}(K^2cd)$ makes it difficult to learn a consistent estimator, as we empirically show in \autoref{sec:abla_statistical_consistency}.
        We solve this problem by computing the high-order moment approximation using the following recursion:
        \begin{align}\label{eq:cascaded_computation}
            \boldsymbol{\phi}_k(\vx) = \boldsymbol{\phi}_{k-1}(\vx) \odot \left(\mW_k^\top \vx\right)
        \end{align}
        This last equation leads to the proposed cascaded architecture for \ourmethod \ summarized in \autoref{algo:horde}.
        We empirically show in \autoref{sec:abla_statistical_consistency} that this recursive approach produces a consistent estimator of the informative high-order moment components.
        
        Then, the \ourmethod \ regularizer consists in computing a DML-like loss function on each of the high-order moments, such that similar (respectively dissimilar) images have similar (respectively dissimilar) high-order moments:
                \begin{align}
            \mathcal{L}_{\ourmethod} = \sum_{k=2}^K \mathcal{L}_k \left(\E_{\vx \sim \mathcal{I}}[\boldsymbol{\phi}_k(\vx)], \E_{\vy \sim \mathcal{J}}[\boldsymbol{\phi}_k(\vy)]\right)
        \end{align}
        In practice, we cannot compute the expectation $\E_{\vx \sim \mathcal{I}}[\boldsymbol{\phi}_k(\vx)]$ since the distribution of $\vx$ is unknown.
        We propose to estimate it using the empirical estimator:
        \begin{align}\label{eq:loss_horde}
            \nonumber \mathcal{L}_{\ourmethod}(\mathcal{I}, \mathcal{J}) &= \sum_{k=2}^K \mathcal{L}_k \left(\frac{1}{\vert\mathcal{I}\vert}\sum_{\vx_i \in \mathcal{I}} \boldsymbol{\phi}_k(\vx_i),\right.\\
            &\quad\qquad\qquad\qquad \left. \frac{1}{\vert\mathcal{J}\vert}\sum_{\vx_j\in\mathcal{J}} \boldsymbol{\phi}_k(\vy_j)\right)
        \end{align}
        where $\{\vx_i \in \mathcal{I}\}$ and $\{\vx_j \in \mathcal{J}\}$ are the sets of deep features extracted from images $\mathcal{I}$ and $\mathcal{J}$.
        
        Hence, the DML model is trained on a combination of a standard DML loss and the \ourmethod \ regularizer on pairs of images $\mathcal{I}$ and $\mathcal{J}$:
        \begin{align}
            \mathcal{L}(\mathcal{I}, \mathcal{J}) &= \mathcal{L}_\textit{DML}(\mathcal{I}, \mathcal{J}) + \mathcal{L}_{\ourmethod}(\mathcal{I}, \mathcal{J})
        \end{align}
        This can easily be extended to any tuple based loss function.
        In practice, we use the same DML loss function for \ourmethod \ ($\forall k, \mathcal{L}_k = \mathcal{L}_\textit{DML}$).
        
        Remark also that at inference time, the image representation $\boldsymbol{\phi}_1(\mathcal{I})$ consists only of the sample mean of the deep features:
        \begin{align}
            \boldsymbol{\phi}_1(\mathcal{I}) = \frac{1}{\vert \mathcal{I}\vert}\sum_{\vx_i \in \mathcal{I}} \vx_i, 
        \end{align}
        and the \ourmethod \ part of the model can be discarded.
        
        \begin{algorithm}[t]
            \caption{High-order moments computation}\label{algo:horde}
            \begin{algorithmic}[1]
                \Require{$\mW_1, \dots, \mW_K$ sampled from $\{-1;+1\}$}
                \Ensure{$K$ first moments approximations}
                \Procedure{ApproxMoments}{$\vx$}
                \State $\boldsymbol{\phi}_2(\vx) \gets \frac{1}{\sqrt{d}}\left( \mW_1^\top \vx \right) \odot \left( \mW_2^\top \vx \right)$
                \State $k \gets 3$
                \While{$k \leq K$}
                    \State $\boldsymbol{\phi}_k(\vx) = \boldsymbol{\phi}_{k-1}(\vx) \odot \left( \mW_k^\top \vx \right)$
                    \State $ k \gets k + 1$
                \EndWhile
                \State \textbf{return} $\boldsymbol{\phi}_2(\vx), \dots, \boldsymbol{\phi}_K(\vx)$
                \EndProcedure
            \end{algorithmic}
        \end{algorithm}
    
    \subsection{Theoretical analysis}\label{sec:met_theory}
        In this section, we show that optimizing distances between high-order moments is directly related to the Maximum Mean Discrepancy (MMD) \cite{Gretton_2007_NIPS} and the Wasserstein distance.
        We consider the Reproducing Kernel Hilbert Space (RKHS) $\mathcal{H}$ of distributions $f: \Omega \mapsto \mathbb{R}^+$ defined on the compact $\Omega \subset \sR^c$, endowed with the Gaussian kernel $k(\vx, \vy) = e^{-\gamma\|\vx-\vy\|^2}$.
        An image is then represented as a distribution $\mathcal{I} \in \mathcal{H}$ from which we can sample a set of deep features $\{ \vx_i \in \Omega \}_i$.
        We denote $\E_{\vx \sim \mathcal{I}}[\vx] \in \sR^c$ the expectation of $\vx$ sampled from $\mathcal{I}$.
        The high-order moments are denoted using their vectorized forms, that is $\E_{\vx \sim \mathcal{I}}[\vx^{\otimes k}] \in \sR^{c^k}$ where $\vx^{\otimes 2} = \vx \otimes \vx, \vx^{\otimes 3} = \vx \otimes \vx \otimes \vx$, \emph{etc.}
        By extension, we use $\E_{\vx \sim \mathcal{I}}[\vx^{\otimes 1}]$ for the mean.
        We assume that all moments exist for every distributions in $\mathcal{H}$ and we note, $\forall \mathcal{I} \in \mathcal{H}$:
        \begin{align}
            \underset{k}{\max} \ \| \ \E_{\vx \sim \mathcal{I}}[\vx^{\otimes k}] \ \|^2 = K < \infty
        \end{align}
        
        Following \cite{Gretton_2007_NIPS}, the MMD between two distributions $\mathcal{I}$ and $\mathcal{J}$ is expressed as:
        \begin{equation}
            \text{MMD}(\mathcal{I},\mathcal{J}) = \underset{T}{\sup} \ \E_{\vx \sim \mathcal{I}}[T(\vx)] - \E_{\vy \sim \mathcal{J}}[T(\vy)]
        \end{equation}
        The MMD searches for a transform $T$ that maximizes the difference between the expectation of two distributions.
        Intuitively, a low MMD implies that both distributions are concentrated in the same regions of the feature space.
        
        In the following theorem, we show that the distance over high-order moments is an upper-bound of the squared MMD (the proof mainly follows \cite{Gretton_2007_NIPS}):
        
        \begin{theorem}\label{th:upper_bound}
            There exists $A \in \sR^{+*}$ such that, for every distributions $\mathcal{I}, \mathcal{J} \in \mathcal{H}$, the MMD is bounded from above by the $p$ first moments of \ $\mathcal{I}$ and $\mathcal{J}$ by:
            \begin{align}
                \nonumber \text{MMD}^2(\mathcal{I}, \mathcal{J}) \leq & \ A \sum_{k=1}^p \left\|\E_{\vx\sim \mathcal{I}}[\vx^{\otimes k}] - \E_{\vy\sim \mathcal{J}}[\vy^{\otimes k}]\right\|^2 \\
                & + 1 + o(\frac{\gamma^p K}{p!})
            \end{align}
        \end{theorem}
        
        \begin{proof}
            As the MMD is a distance on the RKHS $\mathcal{H}$ \cite{Gretton_2007_NIPS}, the square of the MMD can be re-written such as:
            \begin{equation}
                \text{MMD}^2(\mathcal{I},\mathcal{J}) = \| \E_{\vx \sim \mathcal{I}}[\boldsymbol{\phi}(\vx)] - \E_{\vy \sim \mathcal{J}}[\boldsymbol{\phi}(\vy)] \|^2_\mathcal{H}
            \end{equation}
            where $\boldsymbol{\phi}$ is defined using the kernel trick $k(\vx, \vy) = \left< \boldsymbol{\phi}(\vx) \ ; \ \boldsymbol{\phi}(\vy) \right>$.
            Then, we can approximate the Gaussian kernel using its Taylor expansion:
            \begin{align}
                \nonumber k(\vx, \vy) &= \exp({-\gamma\|\vx\|^2 - \gamma\|\vy\|^2}) \exp({2\gamma\left<\vx \ ; \ \vy\right>})  \\
                    \nonumber &= \exp({-\gamma\|\vx\|^2 - \gamma\|\vy\|^2}) \sum_{k=0}^{+\infty} \frac{(2\gamma)^k}{k!} \left< \vx ; \vy\right>^k \\
                     &\leq 1 + \sum_{k=1}^{+\infty} a_k \left< \vx^{\otimes k} ; \vy^{\otimes k} \right>
            \end{align}
            where $a_k =\frac{(2\gamma)^k}{k!} > 0$.
            Thus, we can define $\boldsymbol{\phi}$ as the direct sum of all weighted and vectorized moments:
            \begin{align}\label{eq:phi_kernel_trick}
                \boldsymbol{\phi}(\vx) = \bigoplus_{k=1}^{+\infty} \sqrt{a_k} \vx^{\otimes k}
            \end{align}
            As all moments exist, we can swap the expectation and the direct sum.
            Moreover, since the sequence $a_k = \frac{(2 \gamma)^k}{k!} \xrightarrow{} 0$ when $k \xrightarrow{} + \infty$ and the moments are bounded by $K$, the higher-order moment contributions become negligible compared to the $p$ first moments. 
            Thus, we have:
            \begin{align}
                \nonumber \text{MMD}^2(\mathcal{I},\mathcal{J}) \leq& \ 1 + \sum_{k=1}^{+\infty} a_k \| \E_{\vx\sim \mathcal{I}}[\vx^{\otimes k}]  - \E_{\vy\sim \mathcal{J}}[\vy^{\otimes k}] \|^2\\
                \nonumber \leq& \ A \sum_{k=1}^p \| \E_{\vx\sim \mathcal{I}}[\vx^{\otimes k}]  - \E_{\vx\sim \mathcal{J}}[\vx^{\otimes k}] \|^2 \\
                & +  1 + o(\frac{\gamma^p K}{p!})
            \end{align}
            where $A = \underset{k}{\max} \ a_k$.
        \end{proof}
        
        
    \setlength{\tabcolsep}{6pt}
    \begin{table*}[t!]
        \footnotesize
        \begin{center}
            \begin{tabular}{|c|c|cccccc|cccccc|}\hline
                 & & \multicolumn{6}{c|}{Cub-200-2011} & \multicolumn{6}{c|}{Cars-196} \\\hline
                Backbone & R@ & 1 & 2 & 4 & 8 & 16 & 32 & 1 & 2 & 4 & 8 & 16 & 32 \\\hline\hline
                \multicolumn{14}{|c|}{Loss functions or mining strategies} \\\hline
                \multirow{12}{*}{GoogleNet} & Angular loss \cite{Wang_2017_ICCV} & 54.7 & 66.3 & 76.0 & 83.9 & - & - & 71.4 & 81.4 & 87.5 & 92.1 & - & - \\
                 & HDML \cite{Zheng_2019_CVPR} & 53.7 & 65.7 & 76.7 & 85.7 & - & - & 79.1 & 87.1 & 92.1 & 95.5 & - & - \\
                 & DAMLRMM \cite{Xu_2019_CVPR} & 55.1 & 66.5 & 76.8 & 85.3 & - & - & 73.5 & 82.6 & 89.1 & 93.5 & - & - \\
                 & DVML \cite{Lin_2018_ECCV} & 52.7 & 65.1 & 75.5 & 84.3 & - & - & 82.0 & 88.4 & 93.3 & 96.3 & - & - \\
                 & HTL \cite{Ge_2018_ECCV} & 57.1 & 68.8 & 78.7 & 86.5 & 92.5 & 95.5 & 81.4 & 88.0 & 92.7 & 95.7 & 97.4 & 99.0 \\
                \cline{2-14}
                 & contrastive loss (Ours) & 55.0 & 67.9 & 78.5 & 86.2 & 92.2 & 96.0 & 72.2 & 81.3 & 88.1 & 92.6 & 95.6 & 97.8 \\
                 & contrastive loss + \ourmethod & \underline{57.1} & \underline{69.7} & \underline{79.2} & \underline{87.4} & \underline{92.8} & \underline{96.3} & \underline{76.2} & \underline{85.2} & \underline{90.8} & \underline{95.0} & \underline{97.2} & \underline{98.8} \\
                 & Triplet loss (Ours) & 50.5 & 63.3 & 74.8 & 84.6 & 91.2 & 95.0 & 65.2 & 75.8 & 83.7 & 89.4 & 93.6 & 96.5 \\
                 & Triplet loss + \ourmethod & \underline{53.6} & \underline{65.0} & \underline{76.0} & \underline{85.2} & \underline{91.1} & \underline{95.3} & \underline{74.0} & \underline{82.9} & \underline{89.4} & \underline{93.7} & \underline{96.4} & \underline{98.0} \\
                 & Binomial Deviance (Ours) & 55.9 & 67.6 & 78.3 & 86.4 & 92.3 & 96.1 & 78.2 & 86.0 & 91.3 & 94.6 & 97.1 & 98.3 \\
                 & Binomial Deviance + \ourmethod & \textbf{58.3} &\textbf{ 70.4} & \textbf{80.2 }& \textbf{87.7} & \textbf{92.9} & \textbf{96.3} & \underline{81.5} & \textbf{88.5} & \underline{92.7} & \underline{95.4} & \textbf{97.4} & \underline{98.6} \\
                 & Binomial Deviance + \ourmethod{$^\dag$} & \textbf{59.4} &\textbf{ 71.0} & \textbf{81.0 }& \textbf{88.0} & \textbf{93.1} & \textbf{96.5} & \textbf{83.2} & \textbf{89.6} & \textbf{93.6} & \textbf{96.3} & \textbf{98.0} & \textbf{98.8} \\
                \hline
                \multirow{3}{*}{BN-Inception} & Multi-similarity loss \cite{Wang_2019_CVPR} & 65.7 & 77.0 & \textbf{86.3} & \textbf{91.2} & \textbf{95.0} & 97.3 & 84.1 & 90.4 &94.0 & 96.5 & 98.0 & 98.9 \\
                \cline{2-14}
                 & contrastive loss + \ourmethod & \textbf{66.3} & 76.7 & 84.7 & 90.6 & 94.5 & 96.7 & 83.9 & 90.3 & 94.1 & 96.3 & 98.3 & 99.2 \\
                 & contrastive loss + \ourmethod{$^\dag$} & \textbf{66.8} & \textbf{77.4} & 85.1 & 91.0 & 94.8 & 97.3 & \textbf{86.2} & \textbf{91.9} & \textbf{95.1} & \textbf{97.2} & \textbf{98.5} & \textbf{99.4} \\
                \hline
                \multicolumn{14}{|c|}{Ensemble Methods} \\\hline
                \multirow{7}{*}{GoogleNet} & HDC \cite{Yuan_2017_ICCV} & 53.6 & 65.7 & 77.0 & 85.6 & 91.5 & 95.5 & 73.7 & 83.2 & 89.5 & 93.8 & 96.7 & 98.4 \\
                 & BIER \cite{Opitz_2017_ICCV} & 55.3 & 67.2 & 76.9 & 85.1 & 91.7 & 95.5 & 78.0 & 85.8 & 91.1 & 95.1 & 97.3 & 98.7 \\
                 & A-BIER \cite{Opitz_toap_PAMI} & 57.5 & 68.7 & 78.3 & 86.2 & 91.9 & 95.5 & 82.0 & 89.0 & 93.2 & 96.1 & 97.8 & 98.7 \\
                 & ABE \cite{Kim_2018_ECCV} & 60.6 & 71.5 & 79.8 & 87.4 & - & - & 85.2 & 90.5 & 94.0 & 96.1 & - & - \\
                \cline{2-14}
                 & ABE (Ours) & 60.0 & 71.8 & 81.4 & 88.9 & 93.4 & 96.6 & 79.2 & 87.1 & 92.0 & 95.2 & 97.3 & 98.7 \\
                 & ABE + \ourmethod & \textbf{62.7} & \textbf{74.3} &\textbf{ 83.4} & \textbf{90.2} & \textbf{94.6} & \textbf{96.9} & \textbf{86.4} & \textbf{92.0 }& \textbf{95.3} & \textbf{97.4} &\textbf{ 98.6} & \textbf{99.3} \\
                 & ABE + \ourmethod{$^\dag$} & \textbf{63.9} & \textbf{75.7} &\textbf{84.4} & \textbf{91.2} & \textbf{95.3} & \textbf{97.6} & \textbf{88.0} & \textbf{93.2 }& \textbf{96.0} & \textbf{97.9} &\textbf{ 99.0} & \textbf{99.5} \\
                \hline
            \end{tabular}
            \caption{Comparison with the state-of-the-art on Cub-200-2011 and Cars-196 datasets. Results in percents. {$^\dag$} means that the test scores are computed using all the high-order moments (concatenation + PCA to the embedding size).}
            \label{tab:CUB-CARS}
        \end{center}
        \vspace{-1em}
    \end{table*}
    
    \setlength{\tabcolsep}{7pt}
    \begin{table*}[t!]
        \footnotesize
        \begin{center}
            \begin{tabular}{|c|c|cccc|cccccc|}\hline
                 & & \multicolumn{4}{c|}{Stanford Online Products} & \multicolumn{6}{c|}{In-Shop Clothes Retrieval} \\\hline
                Backbone & R@ & 1 & 10 & 100 & 1000 & 1 & 10 & 20 & 30 & 40 & 50 \\\hline
                \multirow{7}{*}{GoogleNet} & Angular loss \cite{Wang_2017_ICCV} & 70.9 & 85.0 & 93.5 & 98.0 & - & - & - & - & - & - \\
                & HDML \cite{Zheng_2019_CVPR} & 68.7 & 83.2 & 92.4 & - & - & - & - & - & - & - \\
                & DAMLRMM \cite{Xu_2019_CVPR} & 69.7 & 85.2 & 93.2 & - & - & - & - & - & - & - \\
                & DVML \cite{Lin_2018_ECCV} & 70.2 & 85.2 & 93.8 & - & - & - & - & - & - & - \\
                & HTL \cite{Ge_2018_ECCV} & \textbf{74.8} & \textbf{88.3} & \textbf{94.8} & \textbf{98.4} & 80.9 & 94.3 & 95.8 & 97.2 & 97.4 & 97.8\\
                \cline{2-12}
                 & Binomial Deviance (Ours) & 67.4 & 81.7 & 90.2 & 95.4 & 81.3 & 94.2 & 95.9 & 96.7 & 97.2 & 97.6 \\
                & Binomial Deviance + \ourmethod & \underline{72.6} & \underline{85.9} & \underline{93.7} & \underline{97.9} & \textbf{84.4} & \textbf{95.4} & \textbf{96.8} & \textbf{97.4} & \textbf{97.8} & \textbf{98.1} \\
                \hline
                \multirow{2}{*}{BN-Inception} & Multi-similarity loss \cite{Wang_2019_CVPR} & 78.2 & 90.5 & 96.0 & 98.7 & 89.7 & 97.9 & 98.5 & 98.8 & 99.1 & 99.2 \\
                \cline{2-12}
                & contrastive loss + \ourmethod & \textbf{80.1} & \textbf{91.3} & \textbf{96.2} & \textbf{98.7}& \textbf{90.4} & \underline{97.8} & \underline{98.4} & \underline{98.7} & \underline{98.9} & \underline{99.0} \\
                \hline
            \end{tabular}
            \caption{Comparison with the state-of-the-art on Stanford Online Products and In-Shop Clothes Retrieval. Results in percents.}
            \label{tab:SOP-INSHOP}
        \end{center}
        \vspace{-1.5em}
    \end{table*} 
        
        This result implies that regularizing high-order moments to be similar enforces similar images to have deep features sampled from similar distributions.
        Thus, deep features from similar images have a higher probability of being concentrated in the same regions of the feature space.
        
        Next, we show a converse relation between high-order moments and the Wasserstein distance:
        \begin{theorem}
            There exists $a \in \sR^{+*}$ such that, for every distributions $\mathcal{I}, \mathcal{J} \in \mathcal{H}$, the squared Wasserstein distance is bounded from below by the $p$ first moments of \ $\mathcal{I}$ and $\mathcal{J}$ by:
            \begin{align}
                W_1^2(\mathcal{I}, \mathcal{J}) \geq a \sum_{k=1}^p \left\|\E_{\vx\sim \mathcal{I}}[\vx^{\otimes k}] - \E_{\vy\sim \mathcal{J}}[\vy^{\otimes k}]\right\|^2 - o(\frac{\gamma^p}{p!})
            \end{align}
        \end{theorem}
        
        \begin{proof}
            Similarly to the \autoref{th:upper_bound}, we can lower-bound the Gaussian kernel using its Taylor expansion:
            \begin{align}
                \nonumber k(\vx, \vy) \geq \alpha \sum_{k=1}^{+\infty} a_k \left< \vx^{\otimes k} ; \vy^{\otimes k} \right>
            \end{align}
            where $\alpha = \exp(-2 \gamma K)$ and $a_k =\frac{(2\gamma)^k}{k!} > 0$.
            Then, by using the definition of $\boldsymbol{\phi}$ from \autoref{eq:phi_kernel_trick}, a lower-bound for the MMD is:
            \begin{align}\label{eq:lower_bound_mmd}
                \nonumber \text{MMD}^2(\mathcal{I},\mathcal{J}) \geq & \ \alpha \ a' \sum_{k=1}^p \| \E_{\vx\sim \mathcal{I}}[\vx^{\otimes k}]  - \E_{\vy\sim \mathcal{J}}[\vy^{\otimes k}] \|^2 \\
                & - o(\frac{\gamma^p K}{p!})
            \end{align}
            where $a' = \underset{k}{\min} \ a_k$.
            Finally, the MMD is a lower-bound of the Wasserstein distance \cite{Sriperumbudur_2010_JMLR}:
            \begin{align}\label{eq:lower_bound_w_by_mmd}
                \sqrt{K} W_1(\mathcal{I}, \mathcal{J}) \geq \text{MMD}(\mathcal{I}, \mathcal{J})
            \end{align}
            By combining \autoref{eq:lower_bound_mmd} and \autoref{eq:lower_bound_w_by_mmd}, we get the expected lower-bound:
            \begin{align}
                W_1^2(\mathcal{I}, \mathcal{J}) \geq a \sum_{k=1}^p \| \E_{\vx\sim \mathcal{I}}[\vx^{\otimes k}]  - \E_{\vy\sim \mathcal{J}}[\vy^{\otimes k}] \|^2 - o(\frac{\gamma^p}{p!})
            \end{align}
            where $a = \frac{\alpha \ a'}{K}$.
        \end{proof}
        Hence, regularizing high-order moments to be dissimilar enforces dissimilar images to have deep features sampled from different distributions.
        As such, deep features are more distinctive as they are sampled from different regions of the feature space for dissimilar images.
        This is illustrated in \autoref{fig:illustration_horde} ($p=5$) compared to \autoref{fig:illustration_full} ($p=1$).

    
\setlength{\tabcolsep}{2.75pt}
\begin{table*}[t]
    \footnotesize
    \centering
    \begin{tabular}{|c|c|c c|c c c|c c c c|c c c c c|c c c c c c|}\hline
        $k$ & 1 & \multicolumn{2}{c}{2} & \multicolumn{3}{|c}{3} & \multicolumn{4}{|c}{4} & \multicolumn{5}{|c}{5} & \multicolumn{6}{|c|}{6} \\\hline
        $n$ & 1 & 1 & 2 & 1 & 2 & 3 & 1 & 2 & 3 & 4 & 1 & 2 & 3 & 4 & 5 & 1 & 2 & 3 & 4 & 5 & 6 \\\hline
        R@1 & 55.9 & \textbf{57.8} & 58.6 & \underline{56.8} & 58.0 & 56.9 & \textbf{57.8} & 58.8 & 57.6 & 56.1 & \underline{57.4} & 57.7 & 56.8 & 56.3 & 53.3 & \underline{57.4} & 57.9 & 57.1 & 55.6 & 54.4 & 50.7 \\
        R@2 & 67.6 & \underline{69.5} & 70.4 & \underline{68.1} & 69.4 & 68.7 & \underline{69.2} & 70.6 & 70.0 & 68.5 & \underline{68.8} & 69.9 & 69.3 & 68.1 & 65.4 & \textbf{69.9} & 70.6 & 70.5 & 68.9 & 66.2 & 63.0 \\
        R@4 & 78.3 & \underline{79.0} & 79.8 & 78.3 & 78.8 & 78.1 & \underline{78.6} & 79.9 & 79.2 & 78.1 & \underline{78.7} & 78.8 & 79.2 & 78.0 & 75.9 & \textbf{79.4} & 80.0 & 79.9 & 78.7 & 76.5 & 74.0 \\
        R@8 & 86.4 & \underline{86.7} & 87.2 & 86.2 & 86.7 & 86.6 & \underline{86.5} & 87.2 & 87.0 & 85.5 & \textbf{87.0} & 87.1 & 87.1 & 86.5 & 84.2 & \underline{86.9} & 87.4 & 87.4 & 86.7 & 85.4 & 82.5 \\
        \hline
    \end{tabular}
    \caption{Impact of the high order moments as regularizers. We report the Recall@K on CUB. $k$ is the number of chosen orders at training time, and $n$ is the order used at testing time to evaluate the performances. $k=n=1$ is the baseline.}
    \label{tab:abla_reg}
\end{table*}

\begin{table*}[t]
    \footnotesize
    \centering
    \begin{tabular}{|c|c|c c|c c c|c c c c|c c c c c|c c c c c c|}\hline
        k & 1 & \multicolumn{2}{c}{2} & \multicolumn{3}{|c}{3} & \multicolumn{4}{|c}{4} & \multicolumn{5}{|c}{5} & \multicolumn{6}{|c|}{6} \\\hline
        n & 1 & 1 & 2 & 1 & 2 & 3 & 1 & 2 & 3 & 4 & 1 & 2 & 3 & 4 & 5 & 1 & 2 & 3 & 4 & 5 & 6 \\\hline
        R@1 & 55.9 & \underline{57.0} & 53.4 & \underline{57.6} & 54.7 & 50.6 & \underline{57.9} & 55.4 & 52.3 & 47.6 & \underline{58.1} & 55.9 & 53.1 & 48.4 & 43.7 & \textbf{58.4} & 55.7 & 52.9 & 47.8 & 43.9 & 40.5 \\
        R@2 & 67.6 & \underline{68.3} & 65.4 & \underline{69.9} & 67.0 & 63.0 & \underline{69.5} & 67.1 & 65.0 & 60.2 & \textbf{70.3} & 67.7 & 65.0 & 60.8 & 56.0 & \underline{69.9} & 67.6 & 64.9 & 59.9 & 56.0 & 53.0 \\
        R@4 & 78.3 & 78.3 & 75.8 & \underline{79.1} & 76.8 & 73.6 & \underline{79.6} & 77.5 & 75.2 & 71.0 & \textbf{79.9} & 78.2 & 75.5 & 72.8 & 67.2 & \underline{79.8} & 78.0 & 75.6 & 70.2 & 67.2 & 64.7 \\
        R@8 & 86.4 & 86.2 & 84.2 & \underline{87.0} & 84.7 & 82.4 & \underline{87.1} & 85.8 & 83.6 & 80.2 & \underline{87.1} & 85.2 & 83.9 & 81.7 & 78.0 & \textbf{87.3} & 85.6 & 83.8 & 79.6 & 77.5 & 75.2 \\
        \hline
    \end{tabular}
    \caption{Impact of the high order moments when all parameters are trained. We report the Recall@K on CUB. $k$ is the number of chosen orders at training time, and $n$ is the order used at testing time. $k=n=1$ is the baseline.}
    \label{tab:abla_reg_trained}
\vspace{-1em}
\end{table*}

\begin{table*}[t]
    \footnotesize
    \centering
    \begin{tabular}{|c|c|c c|c c c|c c c c|c c c c c|c c c c c c|}\hline
        k & 1 & \multicolumn{2}{c}{2} & \multicolumn{3}{|c}{3} & \multicolumn{4}{|c}{4} & \multicolumn{5}{|c}{5} & \multicolumn{6}{|c|}{6} \\\hline
        n & 1 & 1 & 2 & 1 & 2 & 3 & 1 & 2 & 3 & 4 & 1 & 2 & 3 & 4 & 5 & 1 & 2 & 3 & 4 & 5 & 6 \\\hline
        R@1 & 55.9 & \underline{57.0} & 53.4 & \underline{57.9} & 56.1 & 54.2 & \underline{57.6} & 55.4 & 54.3 & 53.0 & \textbf{58.3} & 56.3 & 56.0 & 54.7 & 52.4 & \underline{57.9} & 56.6 & 55.8 & 55.0 & 53.9 & 51.6 \\
        R@2 & 67.6 & \underline{68.3} & 65.4 & \underline{69.4} & 67.9 & 66.2 & \underline{69.3} & 67.2 & 66.0 & 65.2 & \textbf{70.4} & 68.7 & 68.1 & 66.9 & 64.7 & \underline{69.5} & 68.8 & 68.3 & 67.7 & 65.2 & 64.0 \\
        R@4 & 78.3 & 78.3 & 75.8 & \underline{79.2} & 77.8 & 76.4 & \underline{79.5} & 77.2 & 77.0 & 75.8 & \textbf{80.2} & 78.5 & 78.3 & 76.9 & 75.6 & \underline{79.6} & 76.6 & 77.9 & 77.9 & 75.3 & 74.4 \\
        R@8 & 86.4 & 86.2 & 84.2 & \underline{86.6} & 85.3 & 84.4 & \underline{87.1} & 85.6 & 84.4 & 84.1 & \textbf{87.7} & 86.3 & 86.0 & 85.4 & 84.1 & \underline{87.0} & 86.4 & 85.6 & 84.8 & 84.0 & 83.7 \\
        \hline
    \end{tabular}
    \caption{Impact of the cascaded architecture when all parameters are trained using \autoref{algo:horde}. We report the Recall@K on CUB. $k$ is the number of chosen orders at training time, and $n$ is the order used at testing time. $k=n=1$ is the baseline.}
    \label{tab:abla_reg_cascade_train}
\vspace{-1em}
\end{table*}

\begin{figure*}
    \centering
    \subfloat{\includegraphics[width=0.475\linewidth]{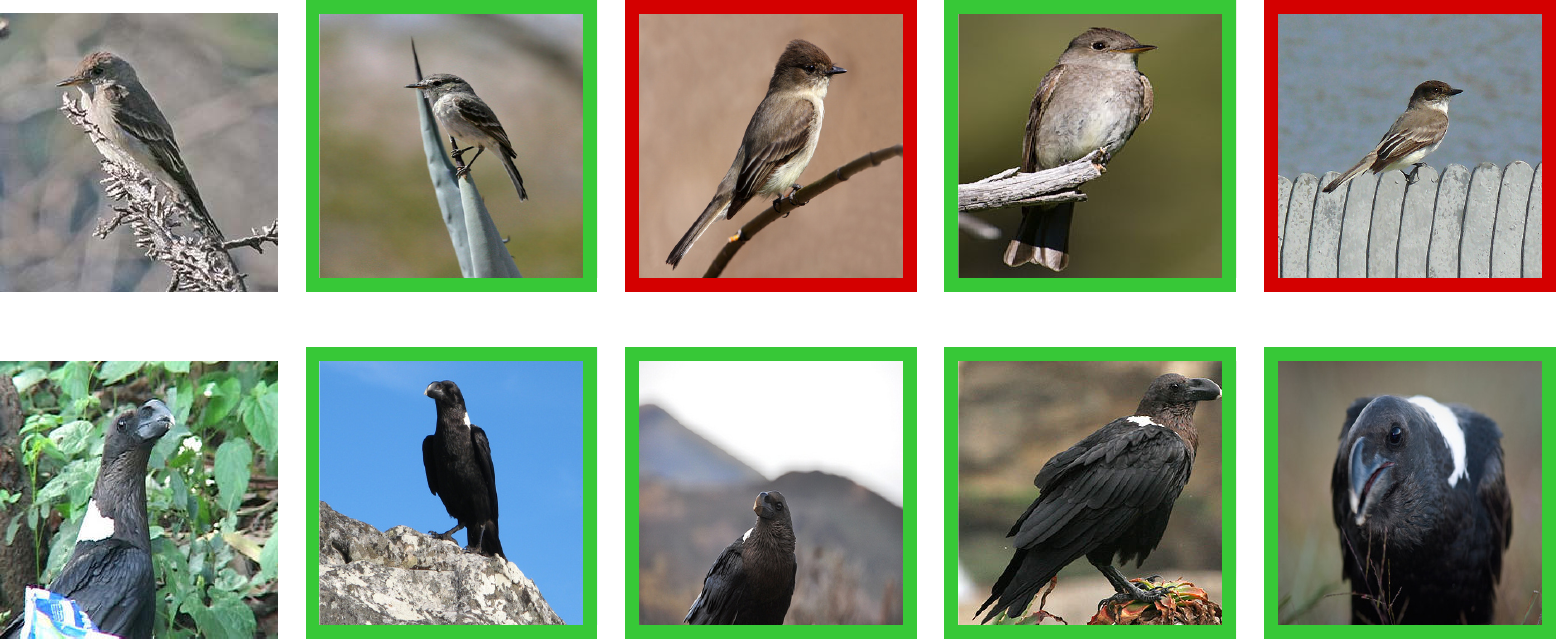}}
    \hfill
    \subfloat{\includegraphics[width=0.475\linewidth]{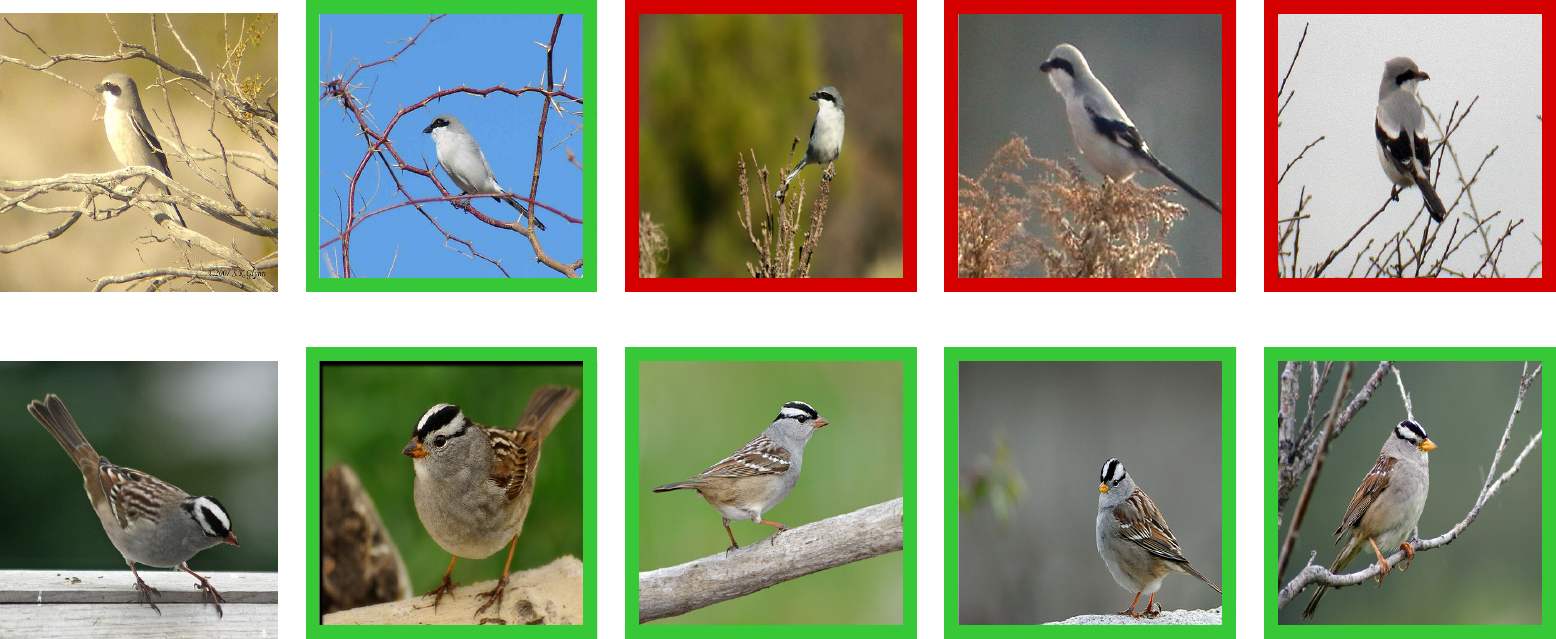}}
    \caption{Qualitative results on CUB for \ourmethod. Correct results are highlighted green (incorrect in red).}
    \label{fig:qualitative_example}
    \vspace{-1em}
\end{figure*} 

\section{Comparison to the state-of-the-art}\label{sec:sota}
    We present the benefits of our method by comparing our results with the state-of-the-art on four datasets, namely CUB-200-2011 (CUB) \cite{CUB_200_2011}, Cars-196 (CARS) \cite{CARS_196}, Stanford Online Products (SOP) \cite{Song_2016_CVPR} and In-Shop Clothes Retrieval (INSHOP) \cite{Liu_2016_CVPR_INSHOP}.
    We report the Recall@K (R@K) on the standard DML splits associated with these datasets.
    Following standard practices, we use GoogleNet \cite{Szegedy_2015_CVPR} as a backbone network and we add a fully connected layer at the end for the embedding.
    For CUB and CARS, we train \ourmethod \ using 5 high-order moments with 5 classes and 8 images per instance per batch.
    For SOP and INSHOP, we use 4 high-order moments with a batch size of 2 images and 40 different classes as there are classes with only 2 images in these datasets.
    We use $256\times256$ crops and the following data augmentation at training time: multi-resolution where the resolution is uniformly sampled in $[80\%, 180\%]$ of the crop size, random crop and horizontal flip.
    At inference time, we only use the images resized to $256\times256$.
    For \ourmethod, we use 8192 dimensions for all high-order moments and we fix all embedding dimensions to 512.
    Finally, we take advantage of the high-order moments at testing time by concatenating them together.
    To be fair with other methods, we reduce their dimensionality to 512 using a PCA.
    These results are annotated with a $^\dag$.
    
    First, we show in the upper part of \autoref{tab:CUB-CARS} that \ourmethod \ significantly improves three popular baselines (contrastive loss, triplet loss and binomial deviance).
    These improvements allow us to claim state of the art results for single model methods on CUB with $\mathbf{58.3}$\% R@1 (compared to $57.1$\% R@1 for HTL \cite{Ge_2018_ECCV}) and second best for CARS.
    
    We also present ensemble method results in the second part of \autoref{tab:CUB-CARS}.
    We show that \ourmethod\ is also a benefit to ensemble methods by improving ABE \cite{Kim_2018_ECCV} by $2.7\%$ R@1 on CUB and $7.2\%$ R@1 on CARS.
    To the best of our knowledge, this allows us to outperform the state of the art methods on both datasets with $\mathbf{62.7}\%$ R@1 on CUB and $\mathbf{86.4}\%$ R@1 on CARS, despite our implementation of ABE under-performing compared to the results reported in \cite{Kim_2018_ECCV}.
    
    Note that both single models and ensemble ones are further improved by using the high-order moments at testing: +1.1\% on CUB and +1.7\% on CARS for the single models + \ourmethod \ and +1.2\% on CUB and +1.6\% on CARS for ABE + \ourmethod.
    
    Furthermore, we show that \ourmethod \ generalizes well to large scale datasets by reporting results on SOP and INSHOP in \autoref{tab:SOP-INSHOP}.
    \ourmethod\ improves our baseline binomial deviance by $5.2\%$ R@1 on SOP and $3.1\%$ R@1 on INSHOP.
    This improvement allows us to claim state of the art results for single model methods on INSHOP with $\mathbf{84.2}\%$ R@1 (compared to $80.9\%$ R@1 for HTL) and second best on SOP with $72.6\%$ R@1 (compared to $74.8\%$ R@1 for HTL).
    Remark also that \ourmethod \ outperforms HTL on 3 out of 4 datasets.
    
    We also report some results with the BN-Inception \cite{Ioffe_2015_ICML}.
    Our model trained with \ourmethod \ and contrastive loss leads to similar results compared to the recent MS loss with mining \cite{Wang_2019_CVPR} on smaller datasets while on larger datasets we outperform it by 1.9\% on SOP and by 0.7\% on INSHOP.
    By using the high-order moments are testing, performances are further increased and outperforms MS loss with mining by 1.1\% on CUB and by 2.1\% on CARS.
    
    Finally, we show some example queries and their nearest neighbors in \autoref{fig:qualitative_example} on the test split of CUB.

\section{Ablation study}\label{sec:abation_studies}
    In this section, we provide an ablation study on the different contributions of this paper.
    We perform 3 experiments on the CUB dataset \cite{CUB_200_2011}.
    The first experiment shows the impact of high-order regularization on a standard architecture while the high-order moments are consistently approximated using the Random Maclaurin approximation.
    The second experiment illustrates the benefit of learning the high-order moments projection matrices.
    The last experiment confirms the statistical consistency of our cascaded architecture when the parameters are learned.
    
    \subsection{Regularization effect}\label{sec:abla_regul_power}
        In this section, we assess the regularization impact of \ourmethod.
        To that aim, we use the baseline detailed in \autoref{sec:sota} and we train the architecture with a number of high-order moments varying from 2 to 6.
        In this first experiment, the computation of the high-order moments does not rely on the cascade computation approach of \autoref{eq:cascaded_computation}.
        Instead, the matrices to approximate the high-order moments are untrainable and sampled using the Random Maclaurin method of \autoref{eq:maclaurin}.
        Remark also that the embedding layers on all high-order moments are not added.
        We use the binomial deviance loss with the standard parameters \cite{Ustinova_2016_NIPS}.
        The results are shown in \autoref{tab:abla_reg}.
        
        First, we can see that \ourmethod \ consistently improves the baseline from 1\% to 2\% in R@1.
        These results corroborate the insights of our theoretical analysis in \autoref{sec:method_overview} and also provide a quantitative evaluation of the behavior observed in \autoref{fig:illustration} on the retrieval ranking.
        When considering the high-order moments as representations, we observe improved results with respect to the baseline for orders 2 and 3.
        Note however that the reported high-order results are not comparable to the first order as the similarity measure is computed on the 8192 dimensional representations.
        While adding orders higher than 2 does not seem interesting in terms of performances, we found that the training process is more stable with 5 or 6 orders than only 2.
        This is observed in practice by measuring the Recall@K with $K \geq 8$ which tend to vary less between training steps.
        Moreover on the CUB dataset, while the baseline requires around 6k steps to reach the best results, we usually need 1k steps less to reach higher accuracy with \ourmethod.
    
    \subsection{Statistical consistency}\label{sec:abla_statistical_consistency}
        To evaluate the impact of estimating only informative high-order moments, we first train the projection matrices and the embeddings but without the cascade architecture and report the results in \autoref{tab:abla_reg_trained}.
        
        In this second experiment, we empirically show that such scheme also increases the baseline by at least 1\% in R@1.
        Notably, by focusing on the most informative high-order moment components, \ourmethod \ further improves the performances of the untrainable \ourmethod\ from 57.8\% to 58.4\%.
        However, the retrieval performances of the high-order representations are heavily degraded compared to \autoref{tab:abla_reg}.
        We interpret these results as an inconsistent estimations of the high-order moments due to overfitting the model.
        For example, the 6\% loss in R@1 for the third-order moment between the first and the second experiments suggests a reduced interest for even higher-order moments.
        
        For the third experiment, we report the results of our cascaded architecture in \autoref{tab:abla_reg_cascade_train}.
        Interestingly, the high-order moments computed from the cascaded architecture perform almost identically to those computed from the untrained method \autoref{tab:abla_reg} but with a smaller dimension.
        Moreover, we keep the performance improvement of the second experiments of \autoref{tab:abla_reg_trained}.
        This confirms that the proposed cascaded architecture does not overfit its estimations of the high-order moments while still improving the baseline.
        Finally, this cascaded architecture only produces a small computational overhead during the training compared to the architecture without the cascade.
        

\section{Conclusion}\label{sec:ccl}
    In this paper, we have presented \ourmethod, a new deep metric learning regularization scheme which improves the distinctiveness of the deep features.
    This regularizer, based on the optimization of the distance between the distributions of the deep features, provides consistent improvements to a wide variety of popular deep metric learning methods.
    We give theoretical insights that show \ourmethod \ upper-bounds the Maximum Mean Discrepancy and lower-bounds the Wasserstein distance.
    The computation of high-order moments is tackled using a trainable Random Maclaurin factorization scheme which is exploited to produce a cascaded architecture with small computation overhead.
    Finally, \ourmethod \ achieves very competitive performances on four well known datasets.
    
\section*{Acknowledgements}
    Authors would like to acknowledge the COMUE Paris Seine University, the Cergy-Pontoise University and M2M Factory for their financial and technical support.

\bibliographystyle{ieee_fullname}
\bibliography{biblio}

\end{document}